\documentclass[runningheads]{llncs}
\usepackage{IEEEtrantools}
\usepackage{microtype}
\usepackage{graphicx}
\usepackage{subfigure}
\usepackage{booktabs} % for professional tables

\usepackage{url}
\usepackage[cmex10]{amsmath}
\usepackage{epstopdf}
\usepackage{amsfonts}
\usepackage{epsfig}
\usepackage{array}
\usepackage{multirow}
\usepackage{color}
\usepackage{amssymb}
\usepackage{bm}
\usepackage{nicefrac}       % compact symbols for 1/2, etc.
\usepackage[mathscr]{euscript}
\usepackage{tabularx}

\spnewtheorem{thm}{Theorem}{\bfseries}{\itshape}
\spnewtheorem{cor}{Corollary}{\bfseries}{\itshape}
\spnewtheorem{lem}{Lemma}{\bfseries}{\itshape}
\spnewtheorem{defn}{Definition}{\bfseries}{\itshape}

\begin{document}
\title{Deep Learning for Inverse Problems: Bounds and Regularizers}
\author{Jaweria Amjad\inst{1} \and
Zhaoyan Lyu\inst{1} \and
Miguel Rodrigues\inst{1}}
\authorrunning{J. Amjad et al.}
\institute{Department of Electronics and Electrical Engineering, University College London\\
\email{\{jaweria.amjad.16,
z.lyu.17,m.rodrigues\}@ucl.ac.uk}}
\maketitle              % typeset the header of the contribution
\begin{abstract}
%{\color{red} Sentence with a brief intro to inverse problems, alluding to some of the approaches typically used to solve such problems}
Inverse problems abound in a number of domains such as medical imaging, remote sensing, and many more, relying on the use of advanced signal \& image processing approaches -- such as sparsity-driven techniques -- to determine their solution.
This paper instead studies the use of deep learning approaches to approximate the solution of inverse problems. In particular, the paper provides a new generalization bound, depending on key quantity associated with a deep neural network -- its Jacobian matrix -- that also leads to a number of computationally efficient regularization strategies applicable to inverse problems %{\color{red} Can we really claim that the regularization approaches are computationally efficient? Zhaoyan -- please comment}.
The paper also tests the proposed regularization strategies in a number of inverse problems including image super-resolution ones. Our numerical results conducted on various datasets show that both fully connected and convolutional neural networks regularized using the regularization or proxy regularization strategies originating from our theory exhibit much better performance than deep networks regularized with standard approaches such as  weight-decay.% {\color{red} we probably also have to specify other regularization strategies used beyond weight-decay}.

\end{abstract}

\section{Introduction}
\label{sec: Intro}

The machine learning community has witnessed a paradigm shift over the past few years, driven in large by the massive progress in the field of deep learning. Deep Neural Networks (DNN) have enjoyed great success in a myriad of applications such as computer vision, natural language processing, speech recognition, speech translation, and many more \cite{goodfellow2014generative},
%This has fueled massive efforts in this domain. In particular, tasks that seemed impossible a few years ago are now fairly tractable given enough data is available to tune a deep neural network \cite{silver2017mastering}.
In particular, tasks that seemed impossible a few years ago are now routinely addressed using DNNs \cite{silver2017mastering}. However, to this day, %scientists do not really understand as to why DNNs \emph{generalize} under highly over parameterized setting.
it is still unclear why deep neural networks perform exceedingly well in various tasks.

The efforts to characterize the %learning behaviour
generalization ability of DNNs -- using tools like VC dimensions and Rademacher compexity -- were further stoked up after \cite{zhang2016understanding} empirically demonstrated the massive capacity of deep learning architectures to fit random labels. This has in turn led to a series of attempts by machine learning theoreticians to explain the generalization properties of DNNs classifiers in terms bounds that rely upon classification margins, parameter norms, sharpness of minima and several other factors \cite{sokolic2017robust,neyshabur2017implicit,neyshabur2018towards,arora2018stronger,jakubovitz2018generalization}

However, despite significant progress made in recent years, the outstanding performance of deep neural networks on regression tasks is less understood than in classification ones. In particular, one important class of regression-oriented tasks where deep learning has been shown to offer outstanding performance improvements in relation to the state-of-the-art -- notably, sparsity-driven techniques -- are inverse problems \cite{mccann2017review}: these typically involve inferring a data vector $\textbf{x}$ from a noisy observation vector $\textbf{y} = \mathcal{F} \left(\textbf{x}\right) + \textbf{n}$ where $\mathcal{F} (\cdot)$ is some linear operator and $\textbf{n}$ is observation noise. It is therefore important to comprehend the factors that control the generalization ability of DNNs in this class of tasks with a view to illuminate further the mechanics undernearth deep learning.
This paper studies the generalization ability of deep neural networks in regression-oriented tasks, with a special emphasis on inverse problems.
Our contributions are as follows:
\begin{itemize}
	\item \emph{{\bf Regression Task-Oriented Generalization Bound:} We provide a bound to the generalization error of a deep neural network based regressor, showing that it depends on quantities such as network Jacobian, sample space complexity, and naturally the sample size. We also specialize such bound applicable to general regression tasks to the important class of inverse problems, showing that it also depends on the Lipchitz constant of the mapping $\mathcal{F}$ and the noise power.}
	\item \emph{{\bf Regularization Strategies for Regression Tasks:} We then highlight a number of  regularization strategies, stemming directly from our generalization bounds. In particular, we also review existing computationally efficient regularization strategies -- which can be seen to be rooted in our bounds -- and likewise we also propose a new orthogonalization based regularization strategy leading to some performance improvements.}
		\item \emph{{\bf Experimental Results:} Finally, we offer a detailed empirical study that showcases that our theoretically-rooted regularization strategies can lead to substantial improvements over traditional ones in inverse problems such as image reconstruction and image super-resolution.} 
\end{itemize}		

The remainder of the paper is organized as follows: In Section \ref{sec: Related Work} 
we overview major works related to our work. We then introduce our problem set-up in Section \ref{sec: Problem}. In Section \ref{sec: GE_Analysis} we study the generalization ability of DNNs on general regression tasks, including inverse problems, and in Section \ref{sec: Reg_Strat} we review and propose various regularization strategies stemming directly from our analysis. In Section \ref{sec: experiments}, we provide simulation results on different neural networks architectures incorporating various regularization strategies. Finally, concluding remarks are drawn in Section \ref{sec: conclusion}. The main proofs are relegated to the Appendix.

A word about notation: We use lower case boldface characters to denote vectors, upper case boldface characters to represent matrices and sets are represented by calligraphic font. For example $\mathbf{x}$ is a vector, $\mathbf{X}$ is a matrix and $\mathcal{X}$ is a set. For a matrix, $\|.\|_{2,2}$ signifies the spectral norm and $\|.\|_{F}$ is used to denote the Frobenious norm. For vectors,  $\|.\|_2$ is used to mean the $\ell_2$ norm.

\section{Related Work}\label{sec: Related Work}

Our work connects to various other works in the literature. In particular, a number of papers have in recent years offered characterizations of the generalization ability of deep neural networks that have in turn inspired new regularization strategies \cite{jakubovitz2018generalization}. For example, both \cite{neyshabur2017implicit} and \cite{bartlett2017spectrally} independently provided $GE$ bounds for deep neural networks, expressed in terms of different norms associated with the collection of network parameters  -- such as group norm, max norm and spectral norm -- thus inspiring new regularizers aiming explicitly at constrain the network complexity by limiting the value of such norms.
%that characterized the network $GE$ in terms of different parameter norms and thus inspired regularizers that constrained network complexity by limiting these norms.
\cite{arora2018stronger} provided a classification framework to characterize the generalization properties of neural networks, leading to linear-algebraic algorithms to effectively limit the number of parameters in individual layers. In \cite{cisse2017parseval}, the authors proposed an upper bound to the $GE$ of a DNN based classifier in the presence of the adversarial perturbations. The authors have also proposed to optimize the network in a manner that forces the weight matrices to remain on the Stiefel manifold \cite{absil2009optimization}, by forcing the gram matrix of the weight matrices to be closer to an identity matrix. Finally, in \cite{sokolic2017robust}, the authors derive a $GE$ bound for large margin DNN classifiers that leads to a new Jacobian regularizer -- involving punishing the Frobenius norm of the network Jacobian -- that also further boosts a deep neural network performance. Our work departs from these works because the focus is on regression in \textit{lieu} of discriminative tasks, despite the fact that some of the regularizers originating from our analysis also connect to some regularizers already proposed in the literature.%

There are various other papers that have in turn suggested a number of new regularization strategies based on empirical considerations, showcasing that such regularization approaches can lead to better performance than conventional ones. In particular \cite{gouk2018regularisation} propose to explicitly enforce an upper bound on the Lipschitz constant of neural networks -- via the operator norm of the weight matrices -- in order to improve its performance. The work puts emphasis on $\ell_1$ and $\ell_\infty$ operator norms, but also showing that the method can be combined with other regularizers such as dropout in order to yield superior cumulative performance.

Various other works \cite{saxe2013exact,mishkin2015all,bansal2018can,cisse2017parseval} have advocated limiting the spectral norm of the weight matrices. Motivated by the norm preservation offered by orthogonal weight matrices, \cite{saxe2013exact} propose orthogonal weight initializations to accelerate the training speed of the neural networks. However, initialization alone does not guarantee orthogonality throughout the training process, hence the orthogonality properties of the final tuned weight matrices may differ substantially from the original ones. Taking this approach a step further, \cite{jia2017improving} propose to initialize the weight matrices by the technique proposed by \cite{saxe2013exact} but simultaneously manually clip the singular values of weight matrices in a narrow window around 1 during the training process in order to maintain orthogonality properties. However, in addition to being computationally expensive owing to the cost associated with the calculation of singular value decompositions (SVD), this method seems counter intuitive since the \emph{new} matrix with the clipped singular values may not be close to the original updated weight matrix, possibly resulting in performance deterioration.  In another work, \cite{brock2016neural} empirically show that regularizing weight matrices to make them orthogonal results in improved performance for generative networks. Motivated by the benefits offered by orthogonal regularization, in \cite{bansal2018can} the authors present a regularization technique for convolution neural networks that forces the weight matrices of a convolutional layer to have a small restricted isometry constant \cite{candes2008restricted}. Other works such as \cite{farnia2018generalizable} and \cite{sedghi2018singular} present efficient algorithms to calculate the spectral norm of the linear transform associated with the convolutional layers. Weight orthogonalization approaches have also received ample attention in recurrent neural networks as well \cite{arjovsky2016unitary,vorontsov2017orthogonality}.

Our work also departs from these works because -- in addition to focus on regression rather than classification tasks -- it also provides a theoretical justification for some of these regularization tasks rooted on generalization error bounds.
\section{Problem Setup}\label{sec: Problem}
We consider the problem of estimating a vector ${\bf x} \in \mathcal{X}\subseteq \mathbb{R}^{N_x}$ from another vector ${\bf y} \in \mathcal{Y}\subseteq \mathbb{R}^{N_y}$ that are related as follows:
\begin{equation}
\label{eq: Inverse_model}
    \mathbf{y} = \mathcal{F}(\mathbf{x})+\mathbf{n}
\end{equation}
where $\mathcal{F}$ is a $L$-Lipschitz continuous operator, i.e. 
	\begin{equation}
	\|\mathcal{F}(\mathbf{x}_1)-\mathcal{F}(\mathbf{x}_2)\|_2\leq L\|\mathbf{x}_1-\mathbf{x}_2\|_2, \forall \mathbf{x}_1,\mathbf{x}_2 \in \mathcal{X}	 \label{mapping}
	\end{equation}
and $\mathbf{n}$ is a bounded perturbation representing noise (i.e. $\|\mathbf{n}\|_2\le \eta$).

This model encapsulates a number of problems arising in practice, including inverse problems such as compressed sensing, image denoising, image deblurring, image super-resolution, and many more \cite{bertero1998introduction}. A number of approaches have been proposed to solve this class of problems including state-of-the-art iteratively reweighted least squares and iterative soft-thresholding methods \cite{chartrand2008iteratively}.

%, where the space $\mathcal{Y}$ consists of a mapping induced by a $L$-Lipschitz continuous linear operator $\mathcal{F}$ on the input space plus a bounded perturbation,i.e.
%	\begin{equation}
%	\mathcal{Y}=\left\{\mathcal{F}(\mathbf{x})+\mathbf{n},\forall \mathbf{x}\in\mathcal{X},\|\mathbf{n}\|_2\le \eta\right\}	 \label{eq" Y_space}
%	\end{equation}
%	where
%	\begin{equation}
%	\|\mathcal{F}(\mathbf{x}_1)-\mathcal{F}(\mathbf{x}_2)\|_2\leq L\|\mathbf{x}_1-\mathbf{x}_2\|_2	 \label{mapping}
%	\end{equation}
%	for any $\mathbf{x}_1,\mathbf{x}_2 \in \mathcal{X}$. We assume for theoretical reasons that both $\mathcal{X}$ and $\mathcal{Y}$ are compact with respect to the $\ell_2$-metric and that the sample space $\mathcal{D} = \mathcal{X} \times \mathcal{Y}$ is compact with respect to the sup-metric $\rho$\footnote{The sup metric $\rho$ on a product space $\mathcal{D}=\mathcal{X}\times\mathcal{Y}$ is given by $\max(\|\mathbf{x}_1-\mathbf{x}_2\|_2,\|\mathbf{y}_1-\mathbf{y}_2\|_2)$, $\forall \mathbf{x}_1,\mathbf{x}_2\in\mathcal{X},  \mathbf{y}_1,\mathbf{y}_2\in\mathcal{Y}$.}.

We consider however a supervised learning approach to solve this problem, involving using a regressor $\mathbf{\Xi}_\mathcal{S}(\cdot): \mathcal{Y} \rightarrow \mathcal{X}$ that has been trained on a set of $m$ examples $\mathcal{S}=\{\mathbf{s}_i=(\mathbf{x}_i,\mathbf{y}_i)\}_{i\leq m}$, drawn independently and identically distributed (IID) from the sample space $\mathcal{D} = \mathcal{X} \times \mathcal{Y}$ according to an unknown distribution $\mu$ underlying the data. 

%We {\color{red} however} solve this problem in a supervised learning fashion where we train a deep learning regressor $\mathbf{\Xi}_\mathcal{S}(\cdot): \mathcal{Y} \rightarrow \mathcal{X}$ on a set of $m$ training examples $\mathcal{S}=\{\mathbf{s}_i=(\mathbf{x}_i,\mathbf{y}_i)\}_{i\leq m}$, drawn IID from the sample space $\mathcal{D} = \mathcal{X} \times \mathcal{Y}$ according to an unknown distribution $\mu$. 

We are interested in characterizing the quality of such a learnt regressor, by assessing how well it will perform on a previously unseen sample ${\bf s}=({\bf x},{\bf y}) \in \mathcal{D}$. This can be done via the generalization error ($GE$) associated with the regressor given by:

\begin{equation}\label{GE}
	GE(\mathbf{\Xi}_\mathcal{S})=|l_{\text{exp}}(\mathbf{\Xi}_\mathcal{S})-l_{\text{emp}}(\mathbf{\Xi}_\mathcal{S})|
\end{equation}
corresponding to the difference between the expected and empirical losses given by:
\begin{equation*}
l_{\text{exp}}(\mathbf{\Xi}_\mathcal{S}) = \mathbb{E}[l(\mathbf{\Xi}_\mathcal{S} (\mathbf{y}),\mathbf{x})]
\end{equation*}
\begin{equation*}
l_{\text{emp}}(\mathbf{\Xi}_\mathcal{S}) = \frac{1}{m}\sum_i l(\mathbf{\Xi}_\mathcal{S} (\mathbf{y}_i),\mathbf{x}_i)
\end{equation*}
where $l(\cdot,\cdot)$ is a loss function that is taken to be the $\ell_2$-loss.
%where the loss function $l(\cdot,\cdot)$ is taken to be the $\ell_2$-loss.

In view of the fact that the class of feed forward DNNs has been shown to deliver outstanding performance in regression problems recently \cite{lucasusing}, we are exclusively interested in characterizing their performance. The output of a $d$-layer feed forward neural network $ \hat{\mathbf{x}} = \mathbf{\Xi}_\mathcal{S}(\mathbf{y})$ given input $\mathbf{y}$ can be expressed as follows: \footnote{A convolutional layer can be represented in a similar setting, where the transformation $\mathbf{W}$ can be represented as a function of doubly-block circulant matrices that in turn are composed of the 4-D filter $\tilde{\mathbf{W}}$ \cite{sedghi2018singular}. }
\begin{equation}
    \hat{\mathbf{x}} = \mathbf{\Xi}_\mathcal{S}(\mathbf{y})=\mathbf{W}_d^T\phi(\ldots\phi(\mathbf{W}_1^T\mathbf{y}+\mathbf{b}_1)\ldots)+\mathbf{b}_d
\end{equation}
where $\mathbf{W}_i$ and $\mathbf{b}_i$ denote the weight matrix and bias vector for the $i$-th layer respectively for all $i\in \{1,\ldots,d\}$ and $\phi(.)$ represents the element-wise activation function such as rectified linear units or sigmoid.

%We concentrate exclusively on $d$-layer feed forward DNNs where we define the output of the DNN as $\mathbf{\Xi}_\mathcal{S}(\mathbf{y})=\mathbf{W}_d^T\phi(\ldots\phi(\mathbf{W_1}_1^T\mathbf{y}+\mathbf{b}_1)\ldots)+\mathbf{b}_d$\footnote{A convolutional layer can be represented in a similar setting, where the transformation $\mathbf{W}$ can be represented as a function of doubly-block circulant matrices that in turn are composed of the 4-D filter $\mathbf{w}$. }. Here $\mathbf{W}_i$ and $\mathbf{b}_i$ denote the weight and bias for the $i$-th layer respectively for all $i\in \{1,\ldots,d\}$ and $\phi(.)$ represents the element-wise activation function.

%We train all our models using Stochastic Gradient Descent (SGD) which iteratively minimizes $l_{emp}$ by taking small steps  towards the direction of its steepest descent of the loss manifold {\color{red} I do not think we need this sentence here}.
%We, next introduce the Jacobian matrix $\mathbf{J}$ of the network. Let the output of the network for an input $\mathbf{x}$ be represented by $\hat{\mathbf{x}}$. Then the Jacobian is given by:

A critical quantity emerging in our analysis is the Jacobian matrix $\mathbf{J}$ of the network given by:
\begin{equation*}
    \mathbf{J} = \frac{d\hat{\mathbf{x}}}{d\mathbf{y}}=\left[\begin{matrix}\frac{\partial\hat{x}_1}{\partial y_1} & \ldots &\frac{\partial\hat{x}_1}{\partial y_{N_y}}\\
    \vdots & \ddots & \vdots\\ \frac{\partial\hat{x}_{N_x}}{\partial y_1} & \ldots &\frac{\partial\hat{x}_{N_x}}{\partial y_{N_y}}\end{matrix}\right]\in \mathbb{R}^{N_x\times N_y}
\end{equation*}

Note that, owing to the chain rule of derivatives, the network Jacobian can further be decomposed into a product of layer-wise Jacobian matrices i.e., $\mathbf{J} = \prod_{i=1}^{d} \mathbf{J}_i$.

The importance of Jacobian matrices -- widely recognized in the context of classification problems -- relates to their role in the stability of the deep neural network, in view of the fact that the magnitude of the singular values of the Jacobian matrix determines the magnitude of signal variations during the forward and backward passes in the network. Therefore, the ability to control the value of the operator norm of the Jacobian matrix also has implications on issues such as vanishing or exploding gradients.

We next showcase -- via an analysis of the generalization gap of DNN regressors -- that the importance of the Jacobian also carries over from classification to regression problems considered here. 

%Jacobian matrices play a significant role in network stability since the singular values of the Jacobian of each layer determines the factor by which the norm of a signal varies in backward or forward pass. This implies that problem of vanishing or exploding gradients can very well be contained if the operator norm of Jacobians are contained to a small value. This discovery together with the fact that the layer wise Jacobians can be estimated by the weight matrices of each layer, has encouraged training techniques that minimize the spectral norm of weights  \cite{yoshida2017spectral,saxe2013exact,jia2017improving,sedghi2018singular}. However, as discussed earlier, none of these papers consider a regression problem. In view of the fact that the generalization ability of deep neural network regressors is poorly understood, our goal in the sequel is to provide general generalization bounds for DNN based regression -- applicable to a wide range of settings -- as well as specialized generalization bounds for DNN based regression applicable to inverse problems. 
	
	\section{Generalization Bounds}\label{sec: GE_Analysis}
Our analysis builds upon the \emph{algorithmic robustness} framework in \cite{xu2012robustness}.
	\begin{defn}\emph{(Algorithmic Robustness)} 
	\label{def:robust}
		Let $\mathcal{S}$ and $\mathcal{D}$ denote the training set and sample space. A learning algorithm is said to be $(K, \epsilon(\mathcal{S}))$-robust if the sample space $\mathcal{D}$ can be partitioned into $K$ disjoint sets $\mathcal{K}_k$, $k = 1,\ldots ,K$,	such that for all $(\mathbf{x}_i,\mathbf{y}_i)\in \mathcal{S}$ and all $(\mathbf{x},\mathbf{y})\in\mathcal{D}$
		\begin{align}
		(\mathbf{x}_i,\mathbf{y}_i), (\mathbf{x},\mathbf{y})\in\mathcal{K}_k
		\implies \\ \nonumber
		\left|l(\mathbf{\Xi}_\mathcal{S} (\mathbf{y}_i),\mathbf{x}_i)-l(\mathbf{\Xi}_\mathcal{S} (\mathbf{y}),\mathbf{x})\right|\leq \epsilon(\mathcal{S}) \label{eq:epsilon}
		\end{align}
	\end{defn}

This notion has already been used to analyse the performance of deep neural networks in \cite{sokolic2017robust,cisse2017parseval}. However, such analyses applicable to classification tasks to not carry over immediately to regression ones.

We begin by re-adapting a relevant result put forth in \cite{sokolic2017robust} that establishes that a deep neural network is Lipschitz continuous with a Lipschitz constant dictated by the network Jacobian. 
%Earlier works \cite{sokolic2017robust}, have exploited this idea to prove generalization bounds for margin based classifiers. In \cite{cisse2017parseval}, algorithmic robustness has been used to provide an upper bound on $GE$ in the presence of adversarial noise. Though insightful for classification tasks, the bounds from these papers cannot be directly used to delineate the performance of deep neural networks in solving inverse problems. However, an important result from \cite{sokolic2017robust} that establishes the Lipschitz continuity condition for deep learning architectures is presented next which serves the basis for our later analysis.
	\begin{lem}
	 \label{Thm : DNN_Lipschitz Con}
		(Adapted from Corollary 2 in \cite{sokolic2017robust}) Consider a $d$-layer DNN based regressor $\mathbf{\Xi}_{\mathcal{S}} (\cdot):\mathcal{Y}\rightarrow\mathcal{X}$. Then, for any $\mathbf{y}_1,\mathbf{y}_2\in\mathcal{Y}$, it follows that
		\begin{equation*}
		\|\mathbf{\Xi}_\mathcal{S}(\mathbf{y}_1)-\mathbf{\Xi}_\mathcal{S}(\mathbf{y}_2)\|_2\leq\prod_{i=1}^{d}\|\mathbf{J}_i\|_{2,2}\|\mathbf{y}_1-\mathbf{y}_2\|_2
		\end{equation*}
		where $\|.\|_{2,2}$ represents the spectral norm of a matrix. 
	\end{lem}
%{\color{red} I do not see the point of this remark. If you want to make some remark, i think we should elaborate here how this theorem differs from the one in Sokolic, but the remark you do seems to be a bit our of place (unless i misunderstood something)} We would like to remark that Corollary 2 in \cite{sokolic2017robust} proved that the Lipschitz constant of the DNN should be bounded for the inputs in the convex hull of $\mathcal{Y}$ which is a hard constraint. However, it was shown in the same work that punishing norm of the network’s Jacobian for the inputs in the training set suffices to regularize the DNN. 

We can now show -- building upon Lemma \ref{Thm : DNN_Lipschitz Con} -- that a DNN regressor is robust.
%Recalling the robustness framework from \cite{xu2012robustness} and exploiting Theorem \ref{Thm : DNN_Lipschitz Con}, we can now show that a DNN regressor, optimized over the $l_2$ loss, is a robust learning algorithm.
	\begin{thm}(Robustness)\label{Thm: Robust_DNN}
		Consider that $\mathcal{X}$ and $\mathcal{Y}$ are compact spaces with respect to the $\ell_2$ metric. Consider also the sample space $\mathcal{D}=\mathcal{X}\times\mathcal{Y}$ equipped with a sup metric \footnote{For $\mathcal{X}$ and $\mathcal{Y}$ compact with $\ell_2$ norm, the sup product metric $\rho(\mathbf{s}_1,\mathbf{s}_2)=\max\{\|\mathbf{x}_1-\mathbf{x}_2\|_2,\|\mathbf{y}_1-\mathbf{y}_2\|_2\}$ for all $\mathbf{s}_1,\mathbf{s}_2\in\mathcal{D}$}. It follows that a $d$-layer DNN based regressor $\mathbf{\Xi}_{\mathcal{S}} (\cdot):\mathcal{Y}\rightarrow\mathcal{X}$ trained on the training set $\mathcal{S}$ is
		\begin{equation*}
		\left(\mathcal{N}\left(\frac{\psi}{2};\mathcal{D},\rho\right),\left(1+\prod_{i=1}^{d}\|\mathbf{J}_i\|_{2,2}\right)\psi\right)-robust
		\end{equation*}
		for any $\psi>0$, where $\mathcal{N}\left(\frac{\psi}{2};\mathcal{D},\rho\right) < \infty$ represents the covering number of the metric space $(\mathcal{D},\rho)$ using metric balls of radius $\psi/2$. 
	\end{thm}
	\begin{proof}
See Appendix.
	\end{proof}

We are now ready to state the main result relating to a $GE$ bound for a robust DNN regressor.
%In the following theorem we derive the $GE$ bound for robust DNN regressors which is the main result of this section.
\begin{thm}(GE Bound)\label{Thm: GE_Bound}
	Consider again that $\mathcal{X}$ and $\mathcal{Y}$ are compact spaces with respect to the $\ell_2$ metric. Consider also the sample space $\mathcal{D}=\mathcal{X}\times\mathcal{Y}$ equipped with a sup metric $\rho$. It follows that a $d$-layer DNN based regressor $\mathbf{\Xi}_{\mathcal{S}} (\cdot):\mathcal{Y}\rightarrow\mathcal{X}$ trained on a training set $\mathcal{S}$ consisting of $m$ i.i.d. training samples obeys with probability $1-\zeta$, for any $\zeta>0$, the $GE$ bound given by:	
	\begin{eqnarray*}
		\nonumber GE(\mathbf{\Xi}_\mathcal{S})
		 &\le &\left(1+\prod_{i=1}^{d}\|\mathbf{J}_i\|_{2,2}\right)\psi+C\left(\mathcal{D},m\right)\\
		 &\le &\left(1+\prod_{i=1}^{d}\|\mathbf{W}_i\|_{2,2}\right)\psi+C\left(\mathcal{D},m\right)\\
		 &\le &\left(1+\prod_{i=1}^{d}\|\mathbf{W}_i\|_{F}\right)\psi+C\left(\mathcal{D},m\right)\label{eq: GE}
	\end{eqnarray*}
	where $C\left(\mathcal{D},m\right)=M\sqrt{\frac{2\mathcal{N}\left(\frac{\psi}{2};\mathcal{D},\rho\right)\log(2)+2\log\left(\frac{1}{\zeta}\right)}{m}}$
	for any $\psi>0$ and $M<\infty$. 
	\end{thm}
\begin{proof}
    See Appendix.
\end{proof}
		
%Our results from Theorems \ref{Thm: Robust_DNN} and \ref{Thm: GE_Bound} provide interesting insights that reinforce some of the to previous theoretical findings in the literature \cite{neyshabur2017implicit,cisse2017parseval}.
Theorems \ref{Thm: Robust_DNN} and \ref{Thm: GE_Bound} also support recent findings in the literature \cite{neyshabur2017implicit,cisse2017parseval}.
For starters, these results suggest that the $GE$ of a deep neural network does not depend directly on the number of network parameters provided that one regularizes appropriately the norm of the weight matrices. Neyshabur \cite{neyshabur2017implicit} has similarly argued that norm based regularization of the weight matrices can improve the generalization ability of a deep neural network in classification tasks. See also \cite{xu2012robustness} that claims that the robustness of a neural entwork does not depend on its size.

% For starters, these results suggest that the $GE$ of a deep neural network is not associated with the number of network parameters and although it apparently seems like these bounds may scale with the network size but this issue can be circumvented with appropriate norm regularization. Xu and Mannor \cite{xu2012robustness} have also shown that robustness of a neural network does not depend on its size. Similarly, in \cite{neyshabur2017implicit}, it is argued that norm based regularization can improve the generalization ability of a deep neural network. 
	
These theorems also suggest that a deeper network may generalize better than a shallower one, by guaranteeing that certain norms of the layer-wise weight matrices are less than one. This result is also aligned with similar claims applicable to classification problems. For example, Neyshabur \cite{neyshabur2017implicit}  put forth similar results deriving from matrix factorization approaches. Several empirical works -- overviewed earlier -- have also shown that constraining the spectral norm of the weight matrices result in better generalization properties.

Theorem \ref{Thm: GE_Bound} also suggests that -- beyond the dependence on the number of training samples -- the generalization ability of a neural network also depends directly on the complexity of the data space $\mathcal{D}$ captured via its covering number. In particular, the $GE$ of more complex data spaces will tend to be higher than the $GE$ of a simpler data space.

Finally, we specialize Theorem \ref{Thm: GE_Bound} from a general regression setting to the inverse problem setting appearing in eq. \eqref{eq: Inverse_model}.
	
	%We, will now specialize the results appearing in Theorems \ref{Thm: Robust_DNN} and \ref{Thm: GE_Bound} to the particular setting of Lipschitz continuous inverse problems.
	\begin{thm}(GE Bound for Inverse Problems)\label{Thm: GE_Bound_Inverse}
	Consider again the spaces $\mathcal{X}$ and $\mathcal{Y}$ equipped with a $\ell_2$ metric, the space  $\mathcal{D}=\mathcal{X}\times\mathcal{Y}$ equipped with the sup-metric $\rho$, and the Lipschitz continuous mapping in \eqref{mapping}. It follows that a $d$-layer DNN based regressor $\mathbf{\Xi}_{\mathcal{S}} (\cdot):\mathcal{Y}\rightarrow\mathcal{X}$ trained on a training set $\mathcal{S}$ consisting of $m$ i.i.d. training samples obeys with probability $1-\zeta$, for any $\zeta>0$, the $GE$ bound given by:
	\begin{eqnarray*}
		\nonumber GE(\mathbf{\Xi}_\mathcal{S})
		 &\le &\left(1+\prod_{i=1}^{d}\|\mathbf{J}_i\|_{2,2}\right)(L\delta+2\eta)+C\left(\mathcal{X},m\right)\\
		 &\le &\left(1+\prod_{i=1}^{d}\|\mathbf{W}_i\|_{2,2}\right)(L\delta+2\eta)+C\left(\mathcal{X},m\right)\\
		 &\le &\left(1+\prod_{i=1}^{d}\|\mathbf{W}_i\|_{F}\right)(L\delta+2\eta)+C\left(\mathcal{X},m\right)\label{eq: RobustApp_Thmr_GE}
	\end{eqnarray*}
	where $C\left(\mathcal{X},m\right)=M\sqrt{\frac{2\mathcal{N}\left(\frac{\delta}{2};\mathcal{X},\ell_2\right)\log(2)+2\log\left(\frac{1}{\zeta}\right)}{m}}$
	for any $\psi>0$, and $M(\mathcal{S})<\infty$. 
    \nonumber
	\end{thm}
	\begin{proof}
		See Appendix.
	\end{proof}

This suggests that -- beyond the number of training samples, the network parameters, and the network depth -- there are three other quantities that affect the generalization error of a deep network in inverse problems: the level of noise, the complexity of the data (via its covering number), the complexity of the mapping (via its Lipschitz constant). In particular, an increase in each of these quantities results as expected in an increase in the generalization error of the network. In general, the deep neural network training process is agnostic to the value of these quantities: it is an interesting question to develop training processes that are cognizant of these quantities with the view to obtain better generalization properties.

%The results embodied in Theorem \ref{Thm: GE_Bound_Inverse} further illuminates an interesting point. When the linear transform is constrained to be Lipschitz continuous, as is the case in classical reconstruction theory \cite{candes2008restricted,donoho2006stable}, then the $GE$-bound can be further controlled by the covering number of the input space $\mathcal{X}$ and the Lipschitz constant $L$. Thus, although deep learning architectures do not usually take domain information such as sparsity into account while solving a problem, this bound shows that for a less complex input space $\mathcal{X}$, we end up with a smaller $GE$ when all the other factors are kept constant.

\section{Computationally Efficient Proxy Regularization Strategies}
\label{sec: Reg_Strat}

Theorems \ref{Thm: GE_Bound} and \ref{Thm: GE_Bound_Inverse} motivate regularization strategies for inverse problems that constrain (1) the spectral norm of the Jacobian matrix, (2) the spectral norm of the weight matrices, or (3) the Frobenious norm of the weight matrices, associated with each individual layer. \footnote{Note that regularization involving the Frobenious norm of the weight matrices is analogous to regularization with weight decay, so we will not discuss this further.}

It is well known that the spectral norm of a matrix can be contained by limiting its maximum singular value. However, it is also known that the cost associated with computing the singular values of a large matrix is a formidable task even for modern software packages.

We will therefore discuss computationally efficient proxy-regularization strategies that also aim to limit the spectral norm of the Jacobian or weight matrices of each individual layer in a deep network. Note that some of these regularization strategies such as Spectral Restricted Isometry Property, Parseval Networks and Spectral Norm regularization have already been proposed in classification or adversarial settings \cite{bansal2018can,cisse2017parseval,yoshida2017spectral} whereas other strategies we discuss appear to be new. 

\subsection{Spectral Norm Regularization}
This technique was proposed by \cite{yoshida2017spectral} to increase the robustness of DNNs to perturbations. The authors proposed to punish the spectral norm of the weight matrices and kernels associated with the fully connected and convolutional layers respectively. The paper uses the power iteration method to efficiently evaluate the spectral norm of the transforms and adds it to the empirical loss of the network, resulting in the following objective function:
\begin{equation}
\label{eq: spec_reg}
    l_{emp}+\beta\sum_i\sigma(\mathbf{W}_i)
\end{equation}
where $\beta$ is the regularization coefficient and $\sigma$ represents the maximum singular value of $\mathbf{W}_i$. \cite{miyato2018spectral} propose a similar techniques, involving the normalization of the singular value of the weight matrix at each layer, ensuring a Lipschitz constant closer to one.%where they propose to normalize the singular value of the weight matrix at each layer and ensuring a Lipschitz constant closer to one. 

In \cite{saxe2013exact}, it is shown theoretically that a stronger condition of ``dynamic symmetry'' which requires all singular values of the layer Jacobian to concentrated around 1 ensures faster convergence for deep linear networks. These results are backed by \cite{wang2016general}, where it was shown through the study of extended data jacobian matrices (EDJM) that the networks for which all the singular values of the Jacobian were closer to its maximum value generalized better. However, only initializing the weight matrices orthogonally does not ensure that the dynamic isometry condition remains fulfilled throughout the training procedure unless the networks are regularized appropriately. Next, we discuss regularization techniques that attempt to achieve dynamic isometry by enforcing orthogonality in the weight matrices.
\subsection{Spectral Restricted Isometry Property (SRIP)} This strategy -- originally proposed in \cite{bansal2018can} -- aims to achieve orthogonality in the filters $\tilde{\mathbf{W}}_i$, of convolutional layers by punishing their Restricted Isometry Constant. It involves optimization of the following objective function: 
\begin{equation}
    l_{emp}+\beta\sum_i\sigma(\tilde{\mathbf{W}}_i^T\tilde{\mathbf{W}}_i-\mathbf{I})
\end{equation}
This regularizer enforces all the singular values of the kernel to be closer to 1 and therefore is a hard constraint in comparison to eq. (\ref{eq: spec_reg}).
\subsection{Weight Orthogonalization (WO)}
We propose to regularize our DNNs using the following constraint:
\begin{equation}
\label{eq: weight_otho_reg}
    l_{emp}+\beta\sum_{i=1}^{d}|\mathbf{W}^T\mathbf{W}-\mathbf{I}|
\end{equation}
where $|.|$ denotes the entry wise $\ell_1$ norm. and is equal to the sum of absolute values in a matrix.

Variations of this regularizers have previously been proposed by \cite{brock2016neural,huang2018orthogonal} for classification tasks. \cite{cisse2017parseval} propose to minimize the Frobenious norm of the difference between the Gram matrix of $\mathbf{W}$ and the Identity matrix. However, instead of adding the regularization term to the empirical loss, they explicitly modify the original gradient of each layer transform. For the experimental results, the authors propose to regularize only a subset $T$ of rows of the weight matrix. Our experimental results showed that this `partially' parseval network fared worse in comparison to networks constrained to have fully orthogonal weight matrices. 

\cite{cisse2017parseval} and more recently, \cite{tsuzuku2018lipschitz} show that, for a convolutional layer with $d_{in}$ input channels, $d_{out}$ channels and a $h\times w$ filter, the Lipschitz constant of the convolutional layer can be bounded in terms of the kernel $\tilde{\mathbf{W}}$ and an arbitrary constant $\alpha$ as $\sqrt{\alpha}\|\tilde{\mathbf{W}}\|_2$. Therefore, for a conv layer, we modify the cost function (\ref{eq: weight_otho_reg}) as follows:
\begin{equation}
\label{eq: weight_otho_reg_conv}
    l_{emp}+\beta\sum_{i=1}^{d}|\tilde{\mathbf{W}}^T\tilde{\mathbf{W}}-\frac{1}{\alpha}\mathbf{I}|
\end{equation}
Following the practice of \cite{yoshida2017spectral,bansal2018can}, we reshape the kernel to $(d_{in}\times h\times h)\times d_{out}$. 

\subsection{Jacobian Orthogonalization (JO)}
For, fully connected networks, we also run experiments for optimizing an objective function that enforces Jacobians at each layer to be orthogonal.  
\begin{equation*}
    l_{emp}+\beta\sum_{i=1}^{d}|\mathbf{J}^T\mathbf{J}-\mathbf{I}|
\end{equation*}
This is a computationally expensive procedure since it requires additional computation of the jacobian for every layer at each training step and empirically it fared worse than weight orthogonal regularizer.

\section{Experimental Results}\label{sec: experiments}
We now conduct a series of experiments that aim to gauge the effectiveness of the various regularizers on a range of inverse problems. In particular, we have considered both fully connected neural networks as well as convolutional neural networks.
%In this section, we conduct experiments to explore different regularizers that achieve our goal of reducing the Lipschitz constant of the network.
\subsection{Fully Connected Network}
{\bf Problem}:
We consider a toy inverse problem involving the reconstruction of an object $\mathbf{x} \in \mathbb{R}^{N_x}$ from linear observations $\mathbf{y} = \mathbf{A} \mathbf{x} \in \mathbb{R}^{N_y}$ where the matrix $\mathbf{A} \in \mathbb{R}^{N_y \times N_x}$ is obtained by sampling $N_x$ column vectors uniformly on the unit sphere in $\mathbb{R}^{N_y}$ with $\nicefrac{N_y}{N_x} = \nicefrac{1}{2}$. This ensures, with high probability, that $\mathbf{A}$ satisfies restricted isometry property and thus is Lipschitz continuous. %In particular, we also consider a compression ratio $\nicefrac{N_y}{N_x}$ of $0.5$.
%For fully connected network, we construct a toy linear inverse problem where the input space $\mathcal{X}$ is considered to be an image dataset. We model the linear function $\mathcal{F}$ by a Lipschitz continuous `fat' transformation matrix $\mathbf{A}\in{\mathbb{R}^{N_y\times N_x}}$ and fix the value of $\eta$ to zero i.e, $\mathcal{Y}=\{\mathbf{Ax},\forall \mathbf{x}\in\mathcal{X}\}$. We obtain $\mathbf{A}$ by sampling $N_x$ column vectors uniformly on the unit sphere in $\mathbb{R}^{N_y}$. This ensures, with high probability, that $\mathbf{A}$ satisfies restricted isometry property and thus is Lipschitz continuous \cite{candes2008restricted}. Lastly, we assume a compression ratio $\nicefrac{N_y}{N_x}$ of $0.5$.

{\bf Dataset Preparation}: We also consider the reconstruction of images given their compressed linear observations, associated with the CIFAR-10 dataset. This dataset is composed of $50K$ training and $10K$ test images where each image is a $32\times32$ natural colour scene with $3$ colour channels. We therefore convert each $32\times32$ image onto a $3072\times1$, we pass such a vectorized image through the compressive linear operator resulting in a $1536\times1$ vector, and we finally reconstruct the original $3072\times1$ vector from the compressed $1536\times1$ one using an appropriately trained neural network.

%We use the images in CIFAR-10 dataset as the input $\mathcal{X}$ and construct the target set $\mathcal{Y}$ using the random matrix $\mathbf{A}$. CIFAR-10 is composed of $50K$ training and $10K$ test images and each image is a $32\times32$ natural colour scene with $3$ colour channels. Without resorting to any preprocessing, we unroll each image, and multiply the resulting $3072\times1$ column vector with the matrix $\mathbf{A}$ to obtain our $1536\times1$ compressed target vector.
\begin{table}[t]
\caption{%$GE$ for various regularizer when a 5-layer feed-forward NN is used to solve a linear inverse problem.
Generalization error and test error for a 5-layer feed-forward NN, employed to solve a toy reconstruction problem, trained using different regularization strategies.
}
\label{Table: FC-table}
\vskip 0.15in
\begin{center}
\begin{small}
\begin{sc}
\begin{tabular}{lcr}
\toprule
Regularizer & Test Loss & $GE$ \\
\midrule
None (Vanilla SGD)   & 3.175 & 0.4284\\
Weight Decay & 2.892 & 0.0298\\
Parseval Network & 2.872    &0.0924\\
SRIP & 2.907    &0.1772\\
Spectral Reg & 2.896     &0.0882\\
WO & 1.986      &0.0100\\
JO  & 2.5371    &0.0394 \\
\bottomrule
\end{tabular}
\end{sc}
\end{small}
\end{center}
\vskip -0.1in
\end{table}

\begin{table}[t]
\caption{%$GE$ for a SRCNN \cite{dong2016image} when its Lipschitz constant is constrained using different techniques.
Generalization error, test error, and PSNR for a SRCNN \cite{dong2016image} , employed to solve an image super-resolution problem, trained using different regularization strategies.}
\label{Table: SRCNN}
\vskip 0.15in
\begin{center}
\begin{small}
\begin{sc}
\begin{tabular}{lccr}
\toprule
Regularizer & Test Loss & $GE$ & PSNR\\
\midrule
None (Vanilla SGD)   & 0.00704 & 0.00110&21.93\\
Weight Decay & 0.00703 & 0.00100& 21.94\\
Parseval Network & 0.00332    &0.00004& 25.49\\
SRIP & 0.0033    &0.00014& 25.58\\
WO & 0.00345      &0.00007&25.33\\
\bottomrule
\end{tabular}
\end{sc}
\end{small}
\end{center}
\vskip -0.1in
\end{table}

{\bf Model and Training}:
We use a fully-connected feed forward DNN with 5 layers with the number of neurons per layer corresponding to the number of dimensions of the target vector (i.e. 3072).%each layer containing neurons equal to the target vector which is 3072 in our case.
We trained this DNN using SGD with an initial learning rate of 0.1 which is then reduced by a factor of half after every 10 epochs. We used mini batches of size 100 and the network was trained for a number of 125 epochs in total.
\subsection{Convolutional Neural Network}
{\bf Problem}:
Here, we consider a classical image super-resolution (SR) problem involving the reconstruction of a high-resolution (HR) image given its low-resolution (LR) version. We closely follow the model and simulation setup adopted by \cite{dong2016image} with few changes discussed later.
\begin{figure}[t]
\vskip 0.2in
\begin{center}
\centerline{\includegraphics[width=\columnwidth]{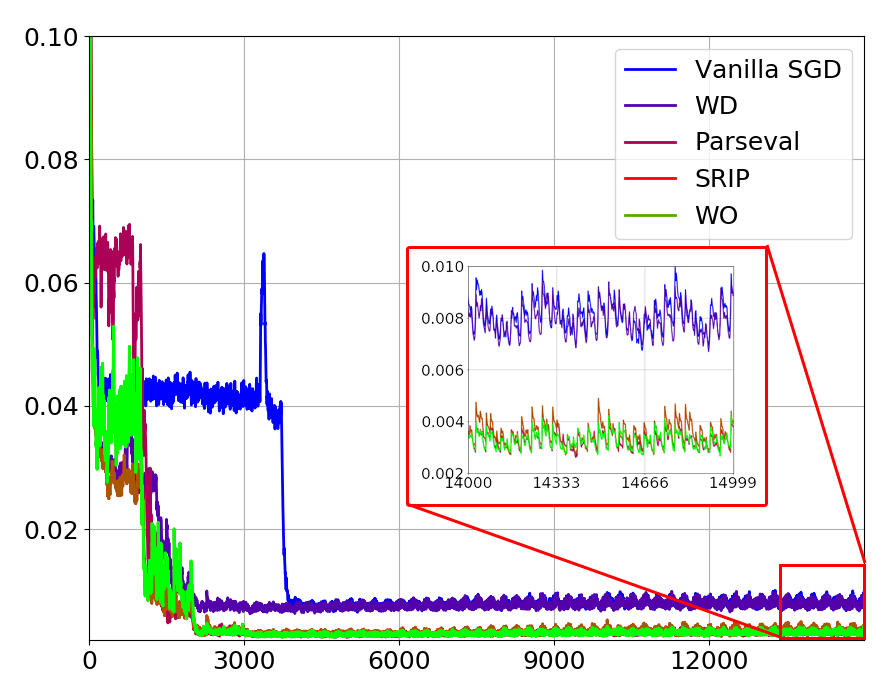}}
\caption{Comparison of learning results for different regularization strategies on SRCNN.}
\label{Fig: SRCNN_learning_curve}
\end{center}
\vskip -0.2in
\end{figure}

{\bf Dataset Preparation}: We train our convolutional model on a dataset consisting of $91$ high resolution coloured images of varying sizes and test the performance of our network on the Set14 dataset \cite{dong2016image}. We crop each image in both sets into different sub images of size $33\times33$ with a stride of $14$ and $28$ for the training and test set, respectively.  Each sub-image image is then downsampled by a factor of $3$ to obtain a LR version of the image. Note that -- as standard in practice -- each LR image is upscaled to the original image size using bicubic interpolation, both during training and testing time, for further processing by the convolutional neural network.

{\bf Model and Training}: We use the 3-layer SRCNN model from \cite{dong2016image} with 64 filters of size $9\times9$ in the first layer, 32 filters of dimension $1\times1$ in the second layer and 3 filters of size $5\times5$ in the third layer respectively.

Here again, we train our models with a minibatch of size 100. We use SGD and with an initial learning rate of 0.01. We train our models for $15000$ epochs and reduce the learning rate by half after every 1000 epochs. 

\subsection{Results}
Table \ref{Table: FC-table} reports results associated with the toy reconstruction problem for different regularization strategies. In particular, we consider a fully connected neural network trained using standard SGD, weight decay, parseval regularization \cite{cisse2017parseval}, spectral norm regularization \cite{yoshida2017spectral}, SRIP \cite{bansal2018can}, WO and JO\footnote{Note that although the results in \cite{bansal2018can} were given for convolutionl networks only, we evaluated the performance of SRIP regularizer for fully conncted layers as well since the analysis holds for fully connected layers as well.}. These results suggest that our proposed proxy regularization strategies WO and JO results in the best generalization error and WO also results in the best test error, compared to all the other competing approaches. The results also suggest that WO can outperform JO -- both in terms of generalization and test error -- despite the fact one might expect that bounding the norm of the Jacobian matrices can lead to better results than bounding the norm of the weight matrices. This result can be explained in part in view of the fact that we are using a proxy regularization strategy.

Table \ref{Table: SRCNN} reports results for the image super-resolution problem. Here, we consider a convolutional neural network regularized with and without weight decay, parseval regularization, SRIP and WO. We did not get good results by regularizing the SRCNN with spectral norm regularization therefore we do not report them here.  Interestingly, our results show that all the regularizers that induced orthogonality in layer transforms performed superior than networks regularized with weight decay; such regularizers also resulted in a very competitive performance with WO outbeating others in terms of $GE$ but only slightly underperforming the others in test error and Peak Signal to Noise Ratio (PSNR).

\section{Conclusions}\label{sec: conclusion}
We have studied the use of deep neural networks in inverse problems. In particular, building on the robustness framework, we have studied the generalization ability of deep learning architectures, offering a generalization bound that encapsulates key quantities associated with the problem such as Lipschitz constant of each layer, covering number of the sample space, noise and smoothness of the mapping between the input and the output space. In particular, 
%We formulate a $GE$ bound for deep learning architectures on inverse problems.
our bound suggests training strategies that result in a network whose per-layer Jacobian matrix exhibits a low spectral norm. We have also explored different proxy regularizer that may result in per-layer Jacobian matrices with low spectral norm, such as Weight Orthogonalization, Jacobian Orthogonalization and SRIP. Our simulation results suggest that both existing and newly proposed proxy regularization strategies can lead to better performance in inverse problems involving image reconstruction or image super-resolution. 
\section*{Appendix}

\begin{proof}[Proof of Theorem \ref{Thm: Robust_DNN}]
	Let $\mathbf{s}_1 = (\mathbf{x}_1,\mathbf{y}_1),\mathbf{s}_2 = (\mathbf{x}_2,\mathbf{y}_2)\in\mathcal{D}$. Then
	\begin{IEEEeqnarray}{rCl}\label{eq : Lip_cont_reg_loss}
	&&|l(\mathbf{\Xi}_\mathcal{S},\mathbf{s}_1)-l(\mathbf{\Xi}_\mathcal{S},\mathbf{s}_2)|\nonumber\\
	&=&\big|\|\mathbf{x}_1-\mathbf{\Xi}_\mathcal{S}(\mathbf{y}_1)\|_2-\|\mathbf{x}_2-\mathbf{\Xi}_\mathcal{S}(\mathbf{y}_2)\|_2\big|\nonumber\\
	&\stackrel{(a)}{\leq}& \|\mathbf{x}_1-\mathbf{\Xi}_\mathcal{S}(\mathbf{y}_1)-\mathbf{x}_2+\mathbf{\Xi}_\mathcal{S}(\mathbf{y}_2)\|_2\nonumber\\
	&\stackrel{(b)}{\leq}& \|\mathbf{x}_1-\mathbf{x}_2\|_2+\|\mathbf{\Xi}_\mathcal{S}(\mathbf{y}_1)-\mathbf{\Xi}_\mathcal{S}(\mathbf{y}_2)\|_2\nonumber\\
	&\stackrel{(c)}{\leq}& \|\mathbf{x}_1-\mathbf{x}_2\|_2+\prod_{i=1}^{d}\|\mathbf{J}_i\|_{2,2}\|\mathbf{y}_1-\mathbf{y}_2\|_2\nonumber\\
	&\stackrel{(d)}{\leq}& \left(1+\prod_{i=1}^{d}\|\mathbf{J}_i\|_{2,2}\right)\rho(\mathbf{s}_1,\mathbf{s}_2)
	\end{IEEEeqnarray}
	The inequalities $(a),(b)$ and $(c)$ hold due to reverse triangle inequality, Minkowski-inequality and Theorem \ref{Thm : DNN_Lipschitz Con}, respectively. Finally, since the sup product metric, $\rho$ upper bounds the distance metric for $\mathcal{X}$ and $\mathcal{Y}$, the inequality $(d)$ is implied.
	
	For a $\nicefrac{\psi}{2}$-cover of $\mathcal{D}$ and $\forall \mathbf{s}_1\in\mathcal{S}\wedge \mathbf{s}_1,\mathbf{s}_2\in\mathcal{D}$, $\rho(\mathbf{s}_1,\mathbf{s}_2)\leq\psi$. Thus
	\begin{IEEEeqnarray*}{c}
	|l(\mathbf{\Xi}_\mathcal{S},\mathbf{s}_1))-l(\mathbf{\Xi}_\mathcal{S},\mathbf{s}_2)|\leq \left(1+\prod_{i=1}^{d}\|\mathbf{J}_i\|_{2,2}\right){\psi}
	\end{IEEEeqnarray*}
	and the theorem follows.
\end{proof}

\begin{proof}[Proof of Theorem \ref{Thm: GE_Bound}]
	Let $\mathcal{D}$ be partitioned into $K$ disjoint sets. The $GE$ of a robust learning algorithm $\mathbf{\Xi}_\mathcal{S})$ is given by \cite{xu2012robustness}:
	\begin{IEEEeqnarray}{rCl}
	\label{eq: GE_Xu}
	GE&\le&|l_{\text{exp}}(\mathbf{\Xi}_\mathcal{S})-l_{\text{emp}}(\mathbf{\Xi}_\mathcal{S})|\nonumber\\
	& \le& \epsilon(\mathcal{S})+|\max_{\mathbf{s}\in\mathcal{D}}l(\mathbf{\Xi}_\mathcal{S},.)|\sqrt{\frac{2K\log(2)+2\log(1/\zeta)}{m}}
		\nonumber\\*
	\end{IEEEeqnarray}
	where $\max_{\mathbf{s}\in\mathcal{D}}l(\mathbf{\Xi}_\mathcal{S},.)$ represents the maximum value of loss over all the samples in the sample space.
	
Let us investigate the various quantities that appear in eq. \ref{eq: GE_Xu}. We know from Theorem \ref{Thm: Robust_DNN} that for a $\nicefrac{\psi}{2}$-cover a $d$-layer DNN, $\mathbf{\Xi}_\mathcal{S}$ is $\left(\left(1+\prod_{i=1}^{d}\|\mathbf{J}_i\|_{2,2}\psi\right),\mathcal{N}(\nicefrac{\psi}{2};\mathcal{D},\rho)\right)$ robust. It can then be shown that the spectral norm of Jacobian matrix of each layer in a DNN can be upper bounded by the spectral norm of the weight matrices \cite{sokolic2017robust,tsuzuku2018lipschitz}. Thus:
	\begin{IEEEeqnarray*}{rCl}
	\label{eq: jacobiannorm_weightnorm}
	    \epsilon(\mathcal{S})\le1+\prod_{i=1}^{d}\|\mathbf{J}_i\|_{2,2}
	    &\le&1+\prod_{i=1}^{d}\|\mathbf{W}_i\|_{2,2}\\
	    &\le&1+\prod_{i=1}^{d}\|\mathbf{W}_i\|_{F}
	\end{IEEEeqnarray*}
	
	Next, we know that for Lipschitz continuous DNNs $\mathbf{\Xi}_{\mathcal{S}}$, the loss function $l(\mathbf{\Xi}_{\mathcal{S}},.)$ is bounded \cite{frechet1904generalisation} and thus for some $M<\infty$, $\max_{\mathbf{s}\in\mathcal{D}}l(\mathbf{\Xi}_\mathcal{S},.)<M$.
	
	Substituting in eq. (\ref{eq: GE_Xu}) gives:
	\begin{IEEEeqnarray*}{rCl}\label{eq:GE_th}
		\nonumber GE(\mathbf{\Xi}_\mathcal{S})&\leq&|l_{\text{exp}}(\mathbf{\Xi}_\mathcal{S})-l_{\text{emp}}(\mathbf{\Xi}_\mathcal{S})|\\
		&{\leq}& \left(1+\prod_{i=1}^{d}\|\mathbf{J}_i\|_{2,2}\right)\psi+C(\mathcal{D},m)\\
		&{\leq}& \left(1+\prod_{i=1}^{d}\|\mathbf{W}_i\|_{2,2}\right)\psi+C(\mathcal{D},m)\\
		&{\leq}& \left(1+\prod_{i=1}^{d}\|\mathbf{W}_i\|_{F}\right)\psi+C(\mathcal{D},m)\\
	\end{IEEEeqnarray*}
	where $C(\mathcal{D},m)= M\sqrt{\frac{2\mathcal{N}(\nicefrac{\psi}{2};\mathcal{D},\rho)\log(2)+2\log(1/\zeta)}{m}}$.
\end{proof}
	\begin{proof}[Proof of Theorem \ref{Thm: GE_Bound_Inverse}]
	We first prove the following Lemma showing that for a Lipschitz continuous mapping $\mathcal{F} (\cdot)$ and sup metric $\rho$, the $\nicefrac{\delta}{2}$-cover of $\mathcal{X}$ upper bounds the $\nicefrac{L\delta}{2}$-cover of $\mathcal{D}$
	%To see why this theorem holds, it suffices to prove the following Lemma which show that for a Lipschitz continuous mapping $\mathcal{F}$ and sup metric $\rho$, the $\nicefrac{\delta}{2}$-cover of $\mathcal{X}$ upper bounds the $\nicefrac{L\delta}{2}$-cover of $\mathcal{D}$.
\begin{lem}\label{Lem : Cov_num_Prod_space}
	Define $\mathcal{F}:\mathbb{R}^{N_x}\rightarrow\mathbb{R}^{N_y}$ to be an $L$-Lipschitz continuous map, namely; $\|\mathcal{F}(\mathbf{x}_1)-\mathcal{F}(\mathbf{x}_2)\|_2\leq L\|\mathbf{x}_1-\mathbf{x}_2\|_2$. Then for the product space $\mathcal{D}=\mathcal{X}\times\mathcal{Y}$, equipped with the metric $\rho$,
	\begin{equation*}
	\mathcal{N}(\nicefrac{L\delta+2\eta}{2};\mathcal{D},\rho)\leq\mathcal{N}(\nicefrac{\delta}{2};\mathcal{X},\ell_2)
	\end{equation*}
\end{lem}	

\begin{proof}	
	Consider a $L$-Lipschitz continuous function $\mathcal{F}$. Let us define the set:% define $\mathcal{Y}$, such that
		\begin{equation*}
	\mathcal{Y}=\{\mathbf{y} = \mathcal{F}(\mathbf{x})+\mathbf{n}: \mathbf{x}\in\mathcal{X},\|\mathbf{n}\|_2\le\eta\}
	\end{equation*}
	Let $\mathcal{X}'$ be a $\delta$-cover of $\mathcal{X}$. Then $\forall \mathbf{x}\in\mathcal{X}$, $\exists\mathbf{x}'\in\mathcal{X}'$ such that $\|\mathbf{x}-\mathbf{x}'\|_2\leq\delta$.
	
	Now, define the set: 
	\begin{equation*}
	\mathcal{Y}'=\{\mathbf{y}' = \mathcal{F}(\mathbf{x}')+\mathbf{n}', \mathbf{x}'\in\mathcal{X}',\|\mathbf{n}'\|_2\le\eta\}
	\end{equation*}
	Then, for any $\mathbf{y} \in \mathcal{Y}$ and $\mathbf{y'} \in \mathcal{Y'}$, the Lipschitz continuity of the mapping implies
	\begin{IEEEeqnarray*}{rCl}
		\|\mathbf{y}-\mathbf{y}'\|_2&\le&\|\mathcal{F}(\mathbf{x})+\mathbf{n}-\mathcal{F}(\mathbf{x}')-\mathbf{n}'\|_2\\
		&\leq&\|\mathcal{F}(\mathbf{x})-\mathcal{F}(\mathbf{x}')\|_2+\|\mathbf{n}-\mathbf{n}'\|_2\\
		&\leq& L\|\mathbf{x}-\mathbf{x}'\|_2+\|\mathbf{n}-\mathbf{n}'\|_2\\
		&\leq& L\delta+2\eta
	\end{IEEEeqnarray*}
	$\mathcal{Y}'$ is, therefore a $(L\delta+2\eta)$-cover of $\mathcal{Y}$.	
	
	Now, define $\mathcal{D}'=\mathcal{X}'\times\mathcal{Y}'$. Then, $\forall \mathbf{s}\in\mathcal{D}, \exists\mathbf{s}'\in\mathcal{D}'$ such that
	\begin{IEEEeqnarray*}{rCl}
		\rho(\mathbf{s},\mathbf{s}')&=&\max(\|\mathbf{x}-\mathbf{x}'\|_2,\|\mathbf{y}-\mathbf{y}'\|_2)\\
		&{\leq}&\max(\delta,L\delta+2\eta)\\
		&=&L\delta+2\eta
	\end{IEEEeqnarray*}		
	Thus $\mathcal{D}'$ is a $(L\delta+2\eta)$-cover of $\mathcal{D}$ 
	This concludes the proof.
\end{proof}	
Theorem \ref{Thm: GE_Bound_Inverse} now follows from Theorem \ref{Thm: GE_Bound} and Lemma \ref{Lem : Cov_num_Prod_space}.
\end{proof}

\bibliographystyle{splncs04}
\bibliography{references_}

\end{document}